\documentclass[runningheads]{llncs}

\usepackage[T1]{fontenc}

\usepackage{amssymb}

\usepackage{mathtools}

\usepackage{xcolor}

\newtheorem{theoremrecallinner}{Theorem}
\newenvironment{theoremrecall}[1]{%
  \renewcommand\thetheoremrecallinner{#1}%
  \theoremrecallinner
}{\endtheoremrecallinner}

\renewcommand*{\proofname}{Proof sketch}

\begin{document}

\title{Modifications of the Miller definition of contrastive (counterfactual) explanations}
\titlerunning{Modifications of the Miller definition of contrastive explanations}

\author{
    Kevin McAreavey
    \and
    Weiru Liu
}
\authorrunning{K. McAreavey and W. Liu}

\institute{
    University of Bristol, United Kingdom \\
    \email{\{ kevin.mcareavey, weiru.liu \}@bristol.ac.uk}
}

\maketitle

\begin{abstract}
	Miller recently proposed a definition of contrastive (counterfactual) explanations based on the well-known Halpern-Pearl (HP) definitions of causes and (non-contrastive) explanations.
	Crucially, the Miller definition was based on the original HP definition of explanations, but this has since been modified by Halpern; presumably because the original yields counterintuitive results in many standard examples.
	More recently Borner has proposed a third definition, observing that this modified HP definition may also yield counterintuitive results.
	In this paper we show that the Miller definition inherits issues found in the original HP definition.
	We address these issues by proposing two improved variants based on the more robust modified HP and Borner definitions.
	We analyse our new definitions and show that they retain the spirit of the Miller definition where all three variants satisfy an alternative unified definition that is modular with respect to an underlying definition of non-contrastive explanations.
	To the best of our knowledge this paper also provides the first explicit comparison between the original and modified HP definitions.
\end{abstract}

\section{Introduction}\label{sec:introduction}
Research on explainable AI (XAI) has seen a massive resurgence in recent years, motivated in large part by concerns over the increasing deployment of opaque machine learning models~\cite{Adadi:IEEEA:2018,Arrieta:IF:2020}.
A common criticism of XAI however is that it exhibits an over-reliance on researcher intuition~\cite{Barocas:FAT:2020,Keane:IJCAI:2021,Miller:AIJ:2019}.
In Miller's seminal survey for the XAI community on insights from philosophy and social science~\cite{Miller:AIJ:2019} he advocates an important theory on human explanations: that they are causal answers to contrastive why-questions.
Subsequently Miller proposed a definition of contrastive explanations~\cite{Miller:KER:2021} based on the well-known Halpern-Pearl (HP) definitions of actual causes~\cite{Halpern:IJCAI:2015} and (non-contrastive) explanations~\cite{Halpern:BJPS:2005b} formalised via structural equation models~\cite{Pearl:book:2000}.
These formal definitions are each designed to capture aspects of causes and explanations as understood in philosophy, social science, and law, e.g.~\cite{Gardenfors:book:1988,Hart:book:1985,Hintikka:book:1962,Lewis:chapter:1986,Lipton:RIPS:1990}.
An interesting research question then is to what extent existing work in XAI satisfies these formal definitions.


Before this kind of theoretical analysis can be fully realised it is necessary to address a crucial limitation of the Miller definition: it is based on a definition of non-contrastive explanations that is now known to yield counterintuitive results in many standard examples.
More precisely, the Miller definition is based on the \emph{original} HP definition of explanations~\cite{Halpern:BJPS:2005b} although a \emph{modified} HP definition has since been proposed by Halpern~\cite{Halpern:book:2016}.
Unlike with his detailed comparisons between the various definitions of actual causes~\cite{Halpern:IJCAI:2015,Halpern:book:2016,Halpern:BJPS:2005a}, Halpern did not explicitly compare these two definitions, so limitations with the original are not entirely clear.
Informally Halpern argues that the modified HP definition yields explanations that correspond more closely to natural language usage; in particular that it captures a notion of \emph{sufficient causes} where the original did not.
More recently, Borner has shown that the modified HP definition may also yield counterintuitive results in some examples~\cite{Borner:BJPS:2021}; he argues that this is due to the modified HP definition failing to fully capture its own notion of sufficient causes, and proposes a third definition that appears to resolve these issues.

The objective of this paper is to lay the groundwork for a theoretical analysis of existing work in XAI with respect to the Miller definition of contrastive explanations.
Firstly, we illustrate why the original HP definition is problematic and argue that the modified HP and Borner definitions offer more robust alternatives; to the best of our knowledge this constitutes the first explicit comparison between the original and modified HP definitions.
Secondly, we show that the Miller definition inherits and even amplifies issues exhibited by the original HP definition.
Thirdly, we address these issues by proposing two improved variants of the Miller definition based on the modified HP and Borner definitions.
Finally, we analyse our new definitions and show that they retain the spirit of the Miller definition where all three variants satisfy an alternative unified definition that is modular with respect to a given definition of non-contrastive explanations.

The rest of this paper is organised as follows:
in Section~\ref{sec:background} we recall the original HP and Miller definitions;
in Section~\ref{sec:non-contrastive} we recall and compare the modified HP and Borner definitions;
in Section~\ref{sec:contrastive} we propose and analyse two new variants of the Miller definition;
and in Section~\ref{sec:conclusion} we conclude.

\section{Background}\label{sec:background}
If $X_{i}$ is a variable, then $\mathcal{D}(X_{i})$ denotes the non-empty set of values of $X_{i}$, called its domain.
If $X = \{ X_{1}, \dots, X_{n} \}$ is a set of variables, then $\mathcal{D}(X) = \mathcal{D}(X_{1}) \times \dots \times \mathcal{D}(X_{n})$ is the domain of $X$.
Each tuple $x = (x_{1}, \dots, x_{n}) \in \mathcal{D}(X)$ is called a setting of $X$ where $x_{i}$ is the value of $X_{i}$ in $x$.

\subsection{Structural equation models}\label{sec:sem}
Here we recall the framework of structural equation models~\cite{Pearl:book:2000}.
A signature is a set $S = U \cup V$ where $U$ is a non-empty set of exogenous variables and $V$ is a non-empty set of endogenous variables with $U \cap V = \emptyset$.
A structural equation for endogenous variable $X_{i}$ is a function $f_{i} : \mathcal{D}(S \setminus \{ X_{i} \}) \to \mathcal{D}(X_{i})$ that defines the value of $X_{i}$ based on the setting of all other variables.
A causal model is a pair $(S, F)$ where $S = U \cup V$ is a signature and $F = \{ f_{i} \mid X_{i} \in V \}$ is a set of structural equations.
A causal setting is a pair $(M, u)$ where $M$ is a causal model and $u$ is a setting of $U$, called a context.
As is common in the literature, we assume that causal models exhibit no cyclic dependencies.
This guarantees a causal setting has a unique solution (i.e.\ a setting of $S$), called the actual world.

For convenience we write e.g.\ $X_{i} \coloneqq \min\{ Z_{j}, Z_{k} \}$ to mean that the structural equation for $X_{i}$ is defined as $f_{i}(z) = \min\{ z_{j}, z_{k} \}$ for each setting $z$ of $Z = S \setminus \{ X_{i} \}$.
If $X \subseteq V$ is a set of endogenous variables and $x$ is a setting of $X$, then $X \leftarrow x$ is an intervention.
An intervention $X \leftarrow x$ on causal model $M = (S, F)$ yields a new causal model $M_{X \leftarrow x} = (S, [ F \setminus F_{X} ] \cup F_{X}')$ where $F_{X} = \{ f_{i} \in F \mid X_{i} \in X \}$ and $F_{X}' = \{ X_{i} \coloneqq x_{i} \mid X_{i} \in X \}$.
In other words, an intervention $X \leftarrow x$ on causal model $M$ replaces the structural equation of each variable $X_{i} \in X$ with a new structural equation that fixes the value of $X_{i}$ to $x_{i}$.

A primitive event is a proposition $X_{i} = x_{i}$ where $X_{i}$ is an endogenous variable and $x_{i}$ is a value of $X_{i}$.
A primitive event $X_{i} = x_{i}$ is true in causal setting $(M, u)$, denoted $(M, u) \models (X_{i} = x_{i})$, if the value of $X_{i}$ is $x_{i}$ in the actual world of $(M, u)$.
An event $\varphi$ is a logical combination of primitive events using the standard logical connectives $\land$, $\lor$, and $\neg$.
Let $(M, u) \models \varphi$ denote that event $\varphi$ is true in causal setting $(M, u)$ by extending entailment of primitive events to logical combinations in the usual way.
A conjunction of primitive events is abbreviated $X = x$, and may be interpreted set-theoretically as its set of conjuncts $\{ X_{i} = x_{i} \mid X_{i} \in X \}$.
We abuse notation in the Boolean case where e.g.\ $X_{1} = 1 \land X_{2} = 0$ may be abbreviated as $X_{1} \land \neg X_{2}$.
If $\varphi$ is an event and $X \leftarrow x$ is an intervention, then $[ X \leftarrow x ] \varphi$ is a causal formula.
A causal formula $[ X \leftarrow x ] \varphi$ is true in causal setting $(M, u)$, denoted $(M, u) \models [ X \leftarrow x ] \varphi$, if $(M_{X \leftarrow x}, u) \models \varphi$.

We refer the reader to~\cite[Chapter 4]{Halpern:book:2016} for a discussion on the expressiveness of structural equation models and their various extensions, including the introduction of norms, typicality, and probability.

\subsection{Non-contrastive causes and explanations}\label{sec:non-constrastive}
Here we recall the HP definition of actual causes~\cite{Halpern:IJCAI:2015} and the original HP definition of explanations~\cite{Halpern:BJPS:2005b}.
It is worth noting that there are in fact three HP definitions of actual causes~\cite{Halpern:IJCAI:2015,Halpern:BJPS:2005a}; the one recalled here is the most recent, known as the \emph{modified} definition, and described by Halpern as his preferred~\cite{Halpern:book:2016}.

\textbf{Actual causes:}
A conjunction of primitive events $X = x$ is an actual cause of event $\varphi$ in causal setting $(M, u)$ if:
\begin{description}
    \item[AC1] $(M, u) \models (X = x) \land \varphi$
    \item[AC2] There is a set of variables $W \subseteq V$ and a setting $x'$ of $X$ such that if $(M, u) \models (W = w)$ then $(M, u) \models [ X \leftarrow x', W \leftarrow w ] \neg \varphi$
    \item[AC3] $X = x$ is minimal relative to AC1--AC2
\end{description}
Intuitively, AC1 requires that $X = x$ and $\varphi$ occur in the actual world of $(M, u)$, while AC3 requires that all conjuncts of an actual cause are necessary.
Perhaps less intuitive is AC2, known as the \emph{but-for} clause~\cite{Hart:book:1985}: but for $X = x$ having occurred, $\varphi$ would not have occurred.
AC2 is what makes this a counterfactual definition of causality, where $\varphi$ is said to counterfactually depend on $X = x$.

Some additional terminology follows.
A conjunction of primitive events $X = x$ is a \textbf{weak} actual cause of event $\varphi$ in causal setting $(M, u)$ if it satisfies AC1 and AC2.
It is worth highlighting that weak actual causes were called sufficient causes in the original HP definition~\cite{Halpern:BJPS:2005a}, which is also the term used by Miller~\cite{Miller:KER:2021}.
In~\cite{Halpern:book:2016} however it is clear that Halpern no longer regards this as an adequate definition of sufficient causes, so we have replaced the term.
We use terms \textbf{partial} and \textbf{part of} to refer to subsets and elements from relevant conjunctions, respectively.
Halpern suggests~\cite{Halpern:book:2016} that an actual cause may be better understood as a complete cause, with each \emph{part of} an actual cause understood simply as a cause.

\textbf{Explanations (original HP definition):}
An epistemic state $K \subseteq \mathcal{D}(U)$ is a set of contexts considered plausible by an agent (called the explainee) prior to observing event $\varphi$.
A conjunction of primitive events $X = x$ is an explanation of event $\varphi$ in causal model $M$ relative to epistemic state $K$ if:
\begin{description}
	\item[EX1] $(M, u) \models \varphi$ for each $u \in K$
	\item[EX2] $X = x$ is a \emph{weak} actual cause of $\varphi$ in $(M, u)$ for each $u \in K$ such that $(M, u) \models (X = x)$
	\item[EX3] $X = x$ is minimal relative to EX2
	\item[EX4] $(M, u) \models \neg (X = x)$ and $(M, u') \models (X = x)$ for some $u, u' \in K$
\end{description}
Intuitively, EX1 requires that the occurrence of \emph{explanandum} $\varphi$ is certain, while EX2 requires that $X = x$ includes an actual cause of $\varphi$ in any plausible context where $X = x$ occurs.
Similar to AC3, EX3 requires that all conjuncts of the explanation are necessary.
Finally, EX4 requires that the occurrence of $X = x$ is uncertain; in effect the explainee can use an explanation to revise its epistemic state by excluding some context(s) previously considered plausible.

\subsection{Contrastive causes and explanations}\label{sec:background:contrastive}
Here we recall the Miller definitions of contrastive causes and contrastive explanations~\cite{Miller:KER:2021}.
The event $\varphi$ in Section~\ref{sec:non-constrastive} is now replaced with a pair $\langle \varphi, \psi \rangle$, called the fact and contrast case, respectively.
Miller considered two contrastive variants: bi-factual ($\psi$ is true) and counterfactual ($\psi$ is false).
In this paper we focus on the counterfactual variants, where $\psi$ is also called the foil.

\textbf{Contrastive causes:}
A pair of conjunctions of primitive events $\langle X = x, X = x' \rangle$ is a contrastive cause of $\langle \varphi, \psi \rangle$ in causal setting $(M, u)$ if:
\begin{description}
	\item[CC1] $X = x$ is a \emph{partial} cause of $\varphi$ in $(M, u)$
	\item[CC2] $(M, u) \models \neg \psi$
	\item[CC3] There is a non-empty set of variables $W \subseteq V$ and a setting $w$ of $W$ such that $X = x'$ is a \emph{partial} cause of $\psi$ in $(M_{W \leftarrow w}, u)$
	\item[CC4] $x_{i} \ne x_{i}'$ for each $X_{i} \in X$
	\item[CC5] $\langle X = x, X = x' \rangle$ is maximal relative to CC1--CC4
\end{description}
Contrastive variants of non-contrastive causes (e.g.\ actual causes) are defined by substituting the word \emph{cause} as appropriate.
Intuitively, CC1 and CC2 require that in the actual world, the fact occurs and the foil does not.
CC1 and CC3 then say that contrastive causes need only reference parts of complete causes that are relevant to both fact and foil.
CC4 captures what Lipton calls the \emph{difference condition}~\cite{Lipton:RIPS:1990}.
Finally, while contrastive causes are not required to reference complete causes, CC5 ensures that information is not discarded unnecessarily.


\textbf{Contrastive explanations (Miller definition):}
A pair of conjunctions of primitive events $\langle X = x, X = x' \rangle$ is a contrastive explanation of $\langle \varphi, \psi \rangle$ in causal model $M$ relative to epistemic state $K$ if:
\begin{description}
	\item[CE1] $(M, u) \models \varphi \land \neg \psi$ for each $u \in K$
	\item[CE2] $\langle X = x, X = x' \rangle$ is a contrastive \emph{weak} actual cause of $\langle \varphi, \psi \rangle$ in $(M, u)$ for each $u \in K$ such that $(M, u) \models (X = x)$
	\item[CE3] $\langle X = x, X = x' \rangle$ is minimal relative to CE2
	\item[CE4]
	\begin{description}
		\item[(a)] $(M, u) \models \neg (X = x)$ and $(M, u') \models (X = x)$ for some $u, u' \in K$
		\item[(b)] There is a non-empty set of variables $W \subseteq V$ and a setting $w$ of $W$ such that $w \ne x$ where $(M_{W \leftarrow w}, u) \models \neg (X = x')$ and $(M_{W \leftarrow w}, u') \models (X = x')$ for some $u, u' \in K$
	\end{description}
\end{description}
Intuitively, contrastive explanations are a natural extension to the original HP definition of (non-contrastive) explanations where there is a direct mapping from CE1--CE4 to EX1--EX4.
The main difference is that explanandum and explanation have been replaced with pairs, capturing fact and foil, with the definition relying on contrastive causes rather than (non-contrastive) causes.



\section{Alternative non-contrastive explanations}\label{sec:non-contrastive}
Here we recall the modified HP definition and compare to the original HP definition, then recall the Borner definition and compare to the modified HP definition.

\subsection{Original HP vs.\ modified HP definition}

\textbf{Explanations (modified HP definition):}
A conjunction of primitive events $X = x$ is an explanation of event $\varphi$ in $M$ relative to epistemic state $K$ if:
\begin{description}
	\item[EX1'] For each $u \in K$:
	\begin{description}
		\item[(a)] There is a conjunct $X_{i} = x_{i}$ that is \emph{part of} an actual cause of $\varphi$ in $(M, u)$ if $(M, u) \models (X = x) \land \varphi$
		\item[(b)] $(M, u) \models [ X \leftarrow x ] \varphi$
	\end{description}
	\item[EX2'] $X = x$ is minimal relative to EX1'
	\item[EX3'] $(M, u) \models (X = x) \land \varphi$ for some $u \in K$
\end{description}
In addition, an explanation is non-trivial if:
\begin{description}
	\item[EX4'] $(M, u) \models \neg (X = x) \land \varphi$ for some $u \in K$ 
\end{description}
According to Halpern~\cite{Halpern:book:2016}, EX1' requires that $X = x$ is a sufficient cause of $\varphi$ in any plausible context where $X = x$ and $\varphi$ occur.
Similar to EX3, EX2' requires that all conjuncts of an explanation are necessary.
EX3' then requires that there is a plausible context where $X = x$ occurs given observation $\varphi$. 
Finally, EX4' requires that the occurrence of $X = x$ is uncertain given $\varphi$.

Three important differences can be observed between the original and modified HP definitions.
Firstly, EX2 requires that an explanation includes an actual cause in all relevant contexts from $K$, while EX1'(a) only requires that an explanation intersects an actual cause in the same contexts.
Secondly, EX1 requires that the occurrence of $\varphi$ is certain, while EX1'(b) only requires that the explanation is sufficient to bring about $\varphi$ in any plausible context.
Thirdly, EX4 requires that the occurrence of $X = x$ is uncertain, while EX4' requires this only for non-trivial explanations.
Together with the view that an actual cause is a complete cause, these observations seem to support the claim that the modified HP definition better captures a notion of sufficient causes.

\begin{example}[Disjunctive forest fire~\cite{Halpern:BJPS:2005a}]\label{ex:disjunctive_forest_fire}
	Consider a causal model with endogenous variables $V = \{ L, \textit{MD}, \textit{FF} \}$ where $L$ is a lightning strike, $\textit{MD}$ is a match being dropped, and $\textit{FF}$ is a forest fire.
	The (Boolean) exogenous variables are $U = \{ U_{L}, U_{\textit{MD}} \}$ and the structural equations are $L \coloneqq U_{L}$, $\textit{MD} \coloneqq U_{\textit{MD}}$, and $\textit{FF} \coloneqq L \lor \textit{MD}$.
	Intuitively, either lightning or match is enough to start the forest fire.
	Let $K$ be the epistemic state containing all contexts.
	The explanandum is $\textit{FF}$.
	The modified HP explanations are (i) $L$, (ii) $\textit{MD}$, and (iii) $\textit{FF}$.
	Both (i) and (ii) are non-trivial.
	Conversely, there are no original HP explanations.
\end{example}

\begin{example}[Overdetermined forest fire~\cite{Borner:BJPS:2021}]\label{ex:overdetermined_forest_fire}
	Consider a variation on the causal model from Example~\ref{ex:disjunctive_forest_fire}.
	The endogenous variables are $V' = V \cup \{ B \}$ where $B$ is a benzine spillage.
	The (Boolean) exogenous variables are $U' = U \cup \{ U_{B} \}$ and the new structural equations are $B \coloneqq U_{B}$ and $\textit{FF} \coloneqq (\textit{MD} \land B) \lor L$.
	Intuitively, lightning is enough to start the forest fire, whereas the match requires that benzine is also present.
	Let $K$ be the epistemic state satisfying $U_{L} \land (\neg U_{\textit{MD}} \lor U_{B})$.
	The explanandum is $\textit{FF}$.
	The modified HP explanations are (i) $L$, (ii) $\textit{MD}$, and (iii) $\textit{FF}$.
	Only (ii) is non-trivial.
	Conversely, the original HP explanations are (i) $L \land \neg \textit{MD}$, (ii) $L \land \textit{MD}$, (iii) $L \land \neg B$, and (iv) $L \land B$. 
\end{example}


Examples~\ref{ex:disjunctive_forest_fire}--\ref{ex:overdetermined_forest_fire} are examples from the literature where the modified HP definition seems more well-behaved than the original.
In Example~\ref{ex:disjunctive_forest_fire} the original HP definition fails to yield any explanations even when an intuitive one exists according to natural language usage, e.g.\ $L$ or $\textit{MD}$ as an explanation of $\textit{FF}$.
Conversely, in Example~\ref{ex:overdetermined_forest_fire} the modified HP definition yields fewer explanations than the original, but this is because the original references seemingly irrelevant information, e.g.\ $L$ on its own is not an explanation of $\textit{FF}$.
These results play out in many other examples, typically with the original HP definition either failing to yield any explanations (because EX1 and EX4 are too strong), or yielding results that contain irrelevant information (because EX2 is problematic).
In Section~\ref{ex:modified_vs_borner} we will elaborate on modified HP explanation (ii) from Example~\ref{ex:overdetermined_forest_fire}.

\subsection{Modified HP vs.\ Borner definition}\label{ex:modified_vs_borner}
Borner examines the modified HP definition based on a closely related definition of sufficient causes proposed by Halpern~\cite{Halpern:book:2016}.
Note that Borner used an earlier HP definition of actual causes but we assume AC1--AC3 throughout.

\textbf{Sufficient causes:}
A conjunction of primitive events $X = x$ is a sufficient cause of event $\varphi$ in causal setting $(M, u)$ if:
\begin{description}
	\item[SC1] $(M, u) \models (X = x) \land \varphi$
	\item[SC2] There is a conjunct $X_{i} = x_{i}$ that is \emph{part of} an actual cause of $\varphi$ in $(M, u)$
	\item[SC3] $(M, u') \models [ X \leftarrow x ] \varphi$ for each $u' \in \mathcal{D}(U)$
	\item[SC4] $X = x$ is minimal relative to SC1--SC3
\end{description}
The difference between weak actual causes (previously called sufficient causes) and sufficient causes (as per this definition) is that the latter replaces AC2 with SC2--SC3, whereas the former simply drops AC3.
As a consequence, weak actual causes are required to include an actual cause, but SC2 only requires a sufficient cause to intersect an actual cause.
The key condition is SC3, which requires that a sufficient cause is sufficient to bring about $\varphi$ in any context.

Clearly the modified HP definition is related to this definition of sufficient causes: EX1'(a) is just SC2 applied to each context from $K$ where $X = x$ and $\varphi$ occur, while EX1'(b) is a weakening of SC3 to contexts from $K$.
The weakening of SC3 can be understood by Halpern's suggestion that it may be too strong to require SC3 to hold even in unlikely contexts~\cite{Halpern:book:2016}.
Borner however suggests~\cite{Borner:BJPS:2021} that the modified HP definition is an attempt to adhere to this guideline: take any $X = x$ where the explainee believes, if $X = x$ is true, then $X = x$ would be a sufficient cause of $\varphi$, but then remove any conjunct of $X = x$ that is (i) already believed or (ii) deducible from $K$ and the remaining conjuncts of $X = x$.
Borner argues that, while either (i) or (ii) may be a reasonable guideline, the modified HP definition appears to arbitrarily alternate between whether it follows (i) or (ii), and this occasionally leads to counterintuitive results.

\textbf{Explanations (Borner definition):}
A conjunction of primitive events $X = x$ is a potential explanation of event $\varphi$ in $M$ relative to epistemic state $K$ if:
\begin{description}
	\item[E1--E2] There is a (possibly empty) conjunction of primitive events $S = s$ with $X \cap S = \emptyset$ such that for each $u \in K$:
	\begin{description}
		\item[(a)] $(X = x) \land (S = s)$ is a sufficient cause of $\varphi$ in $(M, u)$ if $(M, u) \models (X = x)$
		\item[(b)] $(M, u) \models (S = s)$
	\end{description}
	\item[E3--E4] Same as EX3'--EX4'
\end{description}
Alternatively, $X = x$ is an actual explanation if E1--E3 are satisfied and:
\begin{description}
	\item[E5] $(M, u) \models (X = x) \land \varphi$ for each $u \in K$
\end{description}
In addition, a potential explanation is parsimonious if:
\begin{description}
	\item[E6] $X = x$ is minimal relative to E1--E2
\end{description}
%
%
Intuitively, E1--E2(a) requires that $X = x$ is a partial sufficient cause in any plausible context where $X = x$ occurs, while E1--E2(b) requires that $X = x$ should only omit information from a sufficient cause if the occurrence of that information is certain.
E3 is just EX3', while E4 is just EX4'.
For actual explanations, E5 is a much stronger variant of E3, and for parsimonious potential explanations, E6 is comparable to EX2'.
In original and modified HP explanations, explicit minimality clauses serve to exclude irrelevant information.
In potential Borner explanations the minimality clause appears indirectly in E1--E2(a) via SC4.

\begin{example}[Example~\ref{ex:overdetermined_forest_fire} cont.]\label{ex:overdetermined_forest_fire_cont}
	Consider the causal model, epistemic state, and explanandum from Example~\ref{ex:overdetermined_forest_fire}.
	The (parsimonious) potential Borner explanation is $B \land \textit{MD}$.
	The actual Borner explanations are (i) $L$ and (ii) $\textit{FF}$.
	Conversely, recall that the modified HP explanations are (i) $L$, (ii) $\textit{MD}$, and (iii) $\textit{FF}$.
\end{example}


\begin{example}[Suzy and Billy~\cite{Borner:BJPS:2021}]\label{ex:suzy_billy}
	Consider a causal model with endogenous variables $V'' = \{ \textit{SS}, \textit{ST}, \textit{SH}, \textit{BS}, \textit{BT}, \textit{BH}, \textit{BB} \}$ where $\textit{SS}$ is Suzy being sober, $\textit{ST}$ is Suzy throwing, $\textit{SH}$ is Suzy hitting the bottle, and $\textit{BB}$ is the bottle breaking, with $\textit{BS}$, $\textit{BT}$, and $\textit{BH}$ the same for Billy.
	The (Boolean) exogenous variables are $U'' = \{ U_{\textit{SS}}, U_{\textit{ST}}, U_{\textit{BS}}, U_{\textit{BT}} \}$ and the structural equations are 
	$\textit{SS} \coloneqq U_{\textit{SS}}$, 
	$\textit{ST} \coloneqq U_{\textit{ST}}$, 
	$\textit{BS} \coloneqq U_{\textit{BS}}$, 
	$\textit{BT} \coloneqq U_{\textit{BT}}$, 
	$\textit{SH} \coloneqq \textit{SS} \land \textit{ST}$, 
	$\textit{BH} \coloneqq (\textit{BS} \land \textit{BT}\hspace{0.15em}) \land \neg \textit{SH}$, and 
	$\textit{SH} \coloneqq \textit{SH} \hspace{0.15em} \lor \textit{BH}$.
	Intuitively, Suzy and Billy are perfect throwers when sober, although Suzy throws harder than Billy, and if the bottle is hit then it always breaks.
	Let $K$ be the epistemic state satisfying $(U_{\textit{BS}} \land U_{\textit{BT}}) \lor (U_{\textit{BS}} \land U_{\textit{SS}} \land U_{\textit{ST}})$.
	The explanandum is $\textit{BB}$.
	The (parsimonious) potential Borner explanations are 
	(i) $\textit{SS} \land \textit{ST}$, 
	(ii) $\textit{SH}$, and 
	(iii) $\textit{BH}$.
	The actual Borner explanation is $\textit{BB}$.
	Conversely, the modified HP explanations are 
	(i) $\textit{BS} \land \textit{SS}$, 
	(ii) $\textit{BS} \land \textit{ST}$, 
	(iii) $\textit{SS} \land \textit{ST}$, 
	(iv) $\textit{BT} \land \neg \textit{SS}$, 
	(v) $\textit{BT} \land \textit{SS}$, 
	(vi) $\textit{BT} \land \neg \textit{ST}$, 
	(vii) $\textit{BT} \land \textit{ST}$, 
	(viii) $\textit{BT} \land \neg \textit{SH}$, 
	(ix) $\textit{SH}$, 
	(x) $\textit{BH}$, and 
	(xi) $\textit{BB}$.
	All except (xi) are non-trivial.
\end{example}

Example~\ref{ex:overdetermined_forest_fire_cont} says that $\textit{MD}$ is a modified HP explanation of $\textit{FF}$, without reference to $B$, even though both events are required to start the fire.
According to Borner, this is due to the modified HP definition assuming logical omniscience; technically the explainee can deduce $B$ from $K$ given $\textit{MD}$, so the modified HP definition says that $B$ is irrelevant.
Borner argues that this assumption is unrealistic in humans~\cite{Hintikka:book:1962}, with its effect being to ``confront human agents with riddles by only presenting the smallest possible amount of clues by which an agent can in principle deduce the full explanation.''
The Borner definition addresses this by permitting $B \land \textit{MD}$ as a potential explanation.
Comparing Examples~\ref{ex:overdetermined_forest_fire_cont}--\ref{ex:suzy_billy} we see the alternating behaviour of the modified HP definition raised by Borner; $\textit{BS} \land \textit{SS}$ and $\textit{BS} \land \textit{ST}$ are modified HP explanations, although $\textit{BS}$ is deducible from $K$.
In this example the modified HP definition seems unable to discard what amounts to causally irrelevant information in relation to $\textit{BT}$, e.g.\ both $\textit{BT} \land \neg \textit{SS}$ and $\textit{BT} \land \textit{SS}$ are modified HP explanations.
On the other hand, the Borner definition permits no explanation that references $\textit{BT}$, which could support Halpern's view that to require SC3 to hold even in unlikely contexts may be too strong.

\section{Alternative contrastive explanations}\label{sec:contrastive}
Here we propose and analyse two improved variants of the Miller definition based the modified HP and Borner definitions.


\begin{definition}[Contrastive explanations: modified HP variant]\label{def:modified}
	A pair of conjunctions of primitive events $\langle X = x, X = x' \rangle$ is a contrastive explanation of $\langle \varphi, \psi \rangle$ in causal model $M$ relative to epistemic state $K$ if:
	\begin{description}
		\item[CH1] For each $u \in K$:
		\begin{description}
			\item[(a)] There is a pair $\langle X_{i} = x_{i}, X_{i} = x_{i}' \rangle$ that is \emph{part of} a contrastive actual cause of $\langle \varphi, \psi \rangle$ in $(M, u)$ if $(M, u) \models (X = x) \land \varphi \land \neg \psi$
			\item[(b)] There is a pair of (possibly empty) conjunctions of primitive events $\langle S = s, S = s' \rangle$ with $X \cap S = \emptyset$ such that:
			\begin{itemize}
				\item $(M, u) \models [ X \leftarrow x, S \leftarrow s ] \varphi$
				\item There is a non-empty set of variables $W \subseteq V$ and a setting $w$ of $W$ such that $w \ne x$ where $(M_{W \leftarrow w}, u) \models [ X \leftarrow x', S \leftarrow s' ] \psi$
			\end{itemize}
			\item[(c)] $x_{i} \ne x_{i}'$ for each $X_{i} \in X$
			\item[(d)] $\langle X = x, X = x' \rangle$ is maximal relative to CH1(a)--CH1(c)
		\end{description}
		\item[CH2] $\langle X = x, X = x' \rangle$ is minimal relative to CH1
		\item[CH3] $(M, u) \models (X = x) \land \varphi \land \neg \psi$ for some $u \in K$
	\end{description}
	In addition, a contrastive explanation is non-trivial if:
	\begin{description}
		\item[CH4]
		\begin{description}
			\item[(a)] $(M, u) \models \neg (X = x) \land \varphi$ for some $u \in K$
			\item[(b)] There is a non-empty set of variables $W \subseteq V$ and a setting $w$ of $W$ such that $w \ne x$ where $(M_{W \leftarrow w}, u) \models \neg (X = x') \land \psi$ for some $u \in K$
		\end{description}
	\end{description}
\end{definition}

Definition~\ref{def:modified} extends EX1'--EX4' in the spirit of CE1--CE4 by replacing causes with contrastive causes and incorporating the foil $\psi$ as appropriate.
The additional conditions in CH1--CH2 are due to what we already know about modified HP explanations, i.e.\ that they do not perfectly capture sufficient causes and instead integrate a restricted notion of sufficient causes within the definition of explanations itself.
As a consequence, there is no convenient definition of contrastive causes that is adequate for Definition~\ref{def:modified}, and the key characteristics from CC1--CC5 must be integrated directly: 
the use of partial causes in CC1/CC3 is handled by CH1(b), 
the difference condition from CC4 is handled by CH1(c), and 
the maximality condition from CC5 is handled by CH1(d).

\begin{definition}[Contrastive explanations: Borner variant]\label{def:borner}
	A pair of conjunctions of primitive events $\langle X = x, X = x' \rangle$ is a potential contrastive explanation of $\langle \varphi, \psi \rangle$ in causal model $M$ relative to epistemic state $K$ if:
	\begin{description}
		\item[CB1--CB2] There is a pair of (possibly empty) conjunctions of primitive events $\langle S = s, S = s' \rangle$ with $X \cap S = \emptyset$ such that for each $u \in K$:
		\begin{description}
			\item[(a)] $\langle X = x \land S = s, X = x' \land S = s' \rangle$ is a contrastive sufficient cause of $\varphi$ in $(M, u)$ if $(M, u) \models (X = x)$
			\item[(b)] $(M, u) \models (S = s)$
			\item[(c)] There is a non-empty set of variables $W \subseteq V$ and a setting $w$ of $W$ such that $w \ne x$ where $(M_{W \leftarrow w}, u) \models (S = s')$
		\end{description}
		\item[CB3--CB4] Same as CH3--CH4
	\end{description}
	Alternatively, $X = x$ is an actual contrastive explanation if CB1--CB3 and:
	\begin{description}
		\item[CB5] $(M, u) \models (X = x) \land \varphi \land \neg \psi$ for each $u \in K$
	\end{description}
	In addition, a potential contrastive explanation is parsimonious if:
	\begin{description}
		\item[CB6] $\langle X = x, X = x' \rangle$ is minimal relative to CB1--CB2
	\end{description}
\end{definition}

Definition~\ref{def:borner} extends E1--E6 in the spirit of CE1--CE4 by incorporating contrastive causes and the foil as appropriate.
Compared to Definition~\ref{def:modified} this is a more straightforward translation of the non-contrastive definition, since it is able to build on the definitions of contrastive sufficient causes (i.e.\ CC1--CC5 under SC1--SC4).
The biggest change is the inclusion of CB1--CB2(c), which says that E1--E2(b) should also hold for $S = s'$ under an appropriate intervention.

\begin{example}[Example~\ref{ex:overdetermined_forest_fire} cont.]\label{ex:overdetermined_forest_fire_contrastive}
	Consider the causal model and epistemic state from Example~\ref{ex:overdetermined_forest_fire}.\footnote{Note that (contrastive) explanations are not limited to Boolean domains.}
	The explanandum is $\langle \textit{FF}, \neg \textit{FF} \rangle$.
	The Miller contrastive explanations are (i) $\langle \neg \textit{MD}, \textit{MD} \rangle$, (ii) $\langle \textit{MD}, \neg \textit{MD} \rangle$, (iii) $\langle \neg B, B \rangle$, and (iv) $\langle B, \neg B \rangle$.
	The modified HP contrastive explanations are (i) $\langle \textit{MD}, \neg \textit{MD} \rangle$, (ii) $\langle L, \neg L \rangle$, and (iii) $\langle \textit{FF}, \neg \textit{FF} \rangle$.
	Only (i) is non-trivial.
	The potential Borner contrastive explanations are (i) $\langle \textit{MD}, \neg \textit{MD} \rangle$ and (ii) $\langle B, \neg B \rangle$.
	Both (i) and (ii) are parsimonious.
	The actual Borner contrastive explanations are (i) $\langle L, \neg L \rangle$ and (ii) $\langle \textit{FF}, \neg \textit{FF} \rangle$.
\end{example}


Example~\ref{ex:overdetermined_forest_fire_contrastive} demonstrates that the definitions of contrastive explanations inherit characteristics of their corresponding non-contrastive definitions, including the potential for counterintuitive results.
Example~\ref{ex:overdetermined_forest_fire} said that $L \land \neg \textit{MD}$ was an original HP explanations of $\textit{FF}$, which we said was counterintuitive because $\neg \textit{MD}$ was seemingly irrelevant, but in Example~\ref{ex:overdetermined_forest_fire_contrastive} we see that $\langle \neg \textit{MD}, \textit{MD} \rangle$ is an original HP contrastive explanations of $\langle \textit{FF}, \neg \textit{FF} \rangle$, which is not only irrelevant but seemingly nonsensical.
Conversely, Example~\ref{ex:overdetermined_forest_fire_cont} said that $L$ was an explanation of $\textit{FF}$ according to both modified HP and (actual) Borner definitions, and here we see intuitively that $\langle L, \neg L \rangle$ is a contrastive explanation of $\langle \textit{FF}, \neg \textit{FF} \rangle$ for corresponding variants.
These examples suggest that the definitions of contrastive explanations not only inherit, but potentially amplify counterintuitive behaviour exhibited by their corresponding non-contrastive definitions.
There is a clear reason for this: Miller proved that his definition of contrastive explanations was equivalent to an \emph{alternative} definition that depends directly on the original HP definition of non-contrastive explanations~\cite{Miller:KER:2021}.
As we will now show, this alternative Miller definition can in fact be understood as a modular definition of contrastive explanations where equivalence also holds for Definitions~\ref{def:modified}--\ref{def:borner} when instantiated with the appropriate non-contrastive definition.
A trivial generalisation of this alternative Miller definition is as follows:

\begin{definition}[Contrastive explanations: modular variant]\label{def:modular}
	A pair of conjunctions of primitive events $\langle X = x, X = x' \rangle$ is a contrastive explanation of $\langle \varphi, \psi \rangle$ in causal model $M$ relative to epistemic state $K$ {under} a given definition of (non-contrastive) explanations if:
	\begin{description}
		\item[CE1'] $X = x$ is a \emph{partial} explanation of $\varphi$ in $M$ relative to~$K$
		\item[CE2'] There is a non-empty set of variables $W \subseteq V$ and a setting $w$ of $W$ such that $X = x'$ is a \emph{partial} explanation of $\psi$ in $M_{W \leftarrow w}$ relative to $K$
		\item[CE3'] $x_{i} \ne x_{i}'$ for each $X_{i} \in X$
		\item[CE4'] $\langle X = x, X = x' \rangle$ is maximal relative to CE1'--CE3'
	\end{description}
\end{definition}

\begin{theorem}\label{thm:original}
	The Miller definition is equivalent to Definition~\ref{def:modular} {under} original HP explanations if $(M, u) \models \neg \varphi \lor \neg \psi$ for each $u \in K$.
\end{theorem}

\begin{proof}
	The theorem requires that CE1--CE4 is equivalent to CE1'--CE4' {under} EX1--EX4 if the theorem condition holds, i.e.\ $(M, u) \models \neg \varphi \lor \neg \psi$ for each $u \in K$.
	The proof\hspace{0.05em}\footnote{Full proofs are available in the appendix.} is the same as for Theorem~6 in~\cite{Miller:KER:2021} except that the theorem condition makes explicit an assumption that was implicit in Miller's proof; this assumption is required to prove (in the right-to-left case) that if a pair satisfies CE1'--CE4' {under} EX1--EX4 then it must also satisfy CE1 and CC2 via CE2.
	\qed
\end{proof}

\begin{theorem}\label{thm:modified}
	Definition~\ref{def:modified} is equivalent to Definition~\ref{def:modular} {under} modified HP explanations if $(M, u) \models \neg \varphi \lor \neg \psi$ for each $u \in K$.
\end{theorem}

\begin{proof}
	For modified HP contrastive explanations, the theorem requires that CH1--CH3 is equivalent to CE1'--CE4' {under} EX1'--EX3' if the theorem condition holds.
	Non-trivial modified HP contrastive explanations require that CH1--CH4 is equivalent to CE1'--CE4' {under} EX1'--EX4' if the condition holds.
	The proofs follow the same approach as the proof for Theorem~\ref{thm:original}.
	\qed
\end{proof}

\begin{theorem}\label{thm:borner}
	Definition~\ref{def:borner} is equivalent to Definition~\ref{def:modular} {under} Borner explanations if $(M, u) \models \neg \varphi \lor \neg \psi$ for each $u \in K$.
\end{theorem}

\begin{proof}
	For potential Borner contrastive explanations, the theorem requires that CB1--CB4 is equivalent to CE1'--CE4' {under} E1--E4 if the theorem condition holds.
	Actual Borner contrastive explanations require that CB1--CB3, CB5 is equivalent to CE1'--CE4' {under} E1--E3, E5 if the condition holds.
	Parsimonious potential Borner contrastive explanations require that CB1--CB4, CB6 is equivalent to CE1'--CE4' {under} E1--E4, E6 if the condition holds.
	Again the proofs follow the same approach as the proof for Theorem~\ref{thm:original}.
	\qed
\end{proof}

Theorem~\ref{thm:original} is just a restating of the previous result from Miller except that we have added a condition that formalises his assumption of \emph{incompatibility} between fact and foil~\cite[Section 4.1]{Miller:KER:2021}; this assumption is necessary for the theorem to hold.
Theorems~\ref{thm:modified}--\ref{thm:borner} then show that the result generalises for the two other variants.
On the one hand, these theorems demonstrate that Definitions~\ref{def:modified}--\ref{def:borner} successfully capture the modified HP and Borner definitions.
On the other hand, they demonstrate that Definitions~\ref{def:modified}--\ref{def:borner} also successfully capture the Miller definition.
This suggests both a strength and weakness of the Miller definition: Definition~\ref{def:modular} offers a conceptually simpler interpretation of contrastive explanations, yet also highlights the elevated status of non-contrastive explanations.
Therefore it is crucial to choose a non-contrastive definition that is robust.


\section{Conclusion}\label{sec:conclusion}
In this paper we demonstrated generalisability of Miller's definition of contrastive explanations, which was previously bound to the original HP definition of non-contrastive explanations.
We showed that there are at least two other variants of the Miller definition, each derived from a different definition of non-contrastive explanations.
All three variants have a unified modular interpretation {under} the relevant non-contrastive definition.
However, if the underlying non-contrastive definition yields counterintuitive results (as with the original HP definition), then these are inherited by the contrastive definition.
Our new variants address this by supporting non-contrastive definitions that are more robust than the original HP definition without changing the spirit of the Miller definition.
This conclusion implies that future research may focus on developing more robust definitions of non-contrastive explanations, insofar as one accepts Miller's original definition.



We suggested in Section~\ref{sec:introduction} that formal definitions of contrastive explanations offer an interesting foundation for theoretical analyses of existing work in XAI.
One example that is particularly well-suited to such an analysis is an approach to XAI for machine learning known as counterfactual explanations~\cite{Wachter:JOLT:2018,Mothilal:FAT:2020}, which we will abbreviate here as ML-CEs.
The standard formulation~\cite{Verma:ML-RSA:2020} is as follows: if $f : X \to Y$ is a trained classifier, $x \in X$ is a datapoint such that $f(x) = y$ is the prediction for $x$, and $y' \in Y$ is an alternative (counterfactual) prediction such that $y \ne y'$, then an ML-CE is a datapoint $x' \in X$ such that $f(x') = y'$.
The choice of $x'$ from the range of valid ML-CEs is based on some \emph{goodness} heuristic (constructed from e.g.\ Manhattan distance, Hamming distance, or closeness to the data manifold), yet these heuristics often have little theoretical or empirical justification~\cite{Keane:IJCAI:2021}.
Nonetheless, ML-CEs have gained significant prominence in XAI since first introduced in 2018, with recent surveys having identified over 350 papers on the topic~\cite{Keane:IJCAI:2021,Verma:arXiv:2022}.
The particular relevance of ML-CEs to our work is that they can be easily captured by structural equation models; the alteration of feature values from $x$ to $x'$ can then be interpreted as a form of causal attribution (as per the HP definition of actual causes).
What remains then is to understand to what extent these causes can be regarded as explanations, and what role is served by the various heuristics.
Some initial work in this direction has been completed by Mothilal et al.~\cite{Mothilal:AIES:2021} who proposed a bespoke definition of (non-contrastive) explanations and showed that it could provide a unified interpretation of ML-CEs and feature attribution methods in XAI~\cite{Ribeiro:KDD:2016,Lundberg:NIPS:2017}.
However, broader justifications for this bespoke definition remain unclear (e.g.\ applied to standard examples it appears less robust than the HP and Borner definitions), nor does it consider how ML-CEs can be understood as contrastive explanations.

\subsubsection{Acknowledgements}
This work received funding from EPSRC CHAI project (EP/T026820/1).
The authors thank Tim Miller for clarifications on~\cite{Miller:KER:2021}.

\bibliographystyle{splncs04}
\bibliography{ecsqaru2023}

\renewcommand*{\proofname}{Proof}

\section*{Proofs}
Let $\langle X = x, X = x' \rangle$ be a conjunction of primitive events, $\langle \varphi, \psi \rangle$ be a pair of events, $M$ be a causal model, and $K$ be an epistemic state.

\subsection*{Original HP explanations}

\begin{lemma}\label{lem:original}
	CE1--CE4 is equivalent to CE1'--CE4' {under} EX1--EX4 if $(M, u) \models \neg \varphi \lor \neg \psi$ for each $u \in K$.
\end{lemma}

\begin{proof}
	($\rightarrow$)
	If $\langle X = x, X = x' \rangle$ satisfies CE1--CE4 for $\langle \varphi, \psi \rangle$ in $M$ relative to $K$ then $\langle X = x, X = x' \rangle$ satisfies CE1'--CE4' {under} EX1--EX4 in $M$ relative to $K$:
	\begin{description}
		\item[CE1'] This is satisfied if there is some conjunction $(Y = y)$ that satisfies EX1--EX4 for $\varphi$ in $M$ relative to $K$ such that $(X = x) \subseteq (Y = y)$:
		\begin{description}
			\item[EX1] This is satisfied directly by CE1 (irrespective of $X = x$).
			\item[EX2] This is satisfied by CE2 via CC1.
			If CE2 is satisfied then $X = x$ must be a partial weak actual cause of $\varphi$ in $(M, u)$ for each $u \in K$ such that $(M, u) \models (X = x)$.
			Since a partial cause is defined as a subset of a cause, there must be some $Y = y$ that is a weak actual cause of $\varphi$ in $(M, u)$ for each $u \in K$ such that $(M, u) \models (X = x)$.
			\item[EX3] This is satisfied by CE3.
			Suppose $X = x$ is not minimal.
			This implies there is some $Y = y$ satisfying EX2 such that $(Y = y) \subset (X = x)$.
			According to CE2 via CC3 and CC4 there must also be some $Y = y'$ such that $\langle Y = y, Y = y' \rangle$ satisfies CE2.
			It follows that $\langle X = x, X = x' \rangle$ could not be minimal according to CE3.
			This contradicts the premise that $\langle X = x, X = x' \rangle$ satisfies CE1--CE4.
			\item[EX4] This is satisfied by CE2 and CE4(a).
			Firstly, if $(M, u) \models \neg (X = x)$ for some $u \in K$ then it must be that $(M, u) \models \neg (Y = y)$ since $\neg (X = x) \vdash \neg (Y = y)$.
			Secondly, if $Y = y$ is the weak actual cause mentioned in the proof for EX2 then it must be that $Y = y$ satisfies AC1 for some $u' \in K$, i.e.\ $(M, u') \models (Y = y) \land \varphi$ for some $u' \in K$.
		\end{description}
		\item[CE2'] The proof is the same as for CE1' except that $W \leftarrow w$ is the intervention mentioned by CE2 via CC3 and by CE4(b).
		\item[CE3'] This is satisfied directly by CE2 via CC4.
		\item[CE4'] This is satisfied by CE3.
		Suppose $\langle X = x, X = x' \rangle$ is not maximal.
		This implies there is some $\langle Y = y, Y = y' \rangle$ satisfying CE1'--CE3' such that $(X = x) \subset (Y = y)$ and $(X = x') \subset (Y = y')$.
		It follows that $\langle X = x, X = x' \rangle$ could not be maximal according to CE2 via CC5.
		This contradicts the premise that $\langle X = x, X = x' \rangle$ satisfies CE1--CE4.
	\end{description}
	
	($\leftarrow$)
	If $\langle X = x, X = x' \rangle$ satisfies CE1'--CE4' {under} EX1--EX4 for $\langle \varphi, \psi \rangle$ in $M$ relative to $K$ then $\langle X = x, X = x' \rangle$ satisfies CE1--CE4 in $M$ relative to $K$:
	\begin{description}
		\item[CE1] This is satisfied by CE1' (and CE2') via EX1 combined with the premise of this lemma.
		If $(M, u) \models \varphi$ and $(M, u) \models \neg \varphi \lor \neg \psi$ for each $u \in K$ then it must be that $(M, u) \models \varphi \land \neg \psi$ for each $u \in K$  since $\varphi \land (\neg \varphi \lor \neg \psi) \vdash (\varphi \land \neg \psi)$. 
		\item[CE2] This is satisfied if $\langle X = x, X = x' \rangle$ satisfies CC1--CC5 for $\langle \varphi, \psi \rangle$ in $(M, u)$ for each $u \in K$ such that $(M, u) \models (X = x)$, except that references to actual causes in CC1--CC5 are replaced by weak actual causes:
		\begin{description}
			\item[CC1] This is satisfied directly by CE1' via EX2.
			\item[CC2] This is satisfied by CE1' via EX1 combined with the premise of this lemma.
			If $(M, u) \models \varphi$ and $(M, u) \models \neg \varphi \lor \neg \psi$ then it must be that $(M, u) \models \neg \psi$ since $\varphi \land (\neg \varphi \lor \neg \psi) \vdash \neg \psi$.
			\item[CC3] This is satisfied directly by CE2' via EX2.
			\item[CC4] This is satisfied directly by CE3'.
			\item[CC5] This is satisfied by CE4'.
			Suppose $\langle X = x, X = x' \rangle$ is not maximal.
			This implies there is some $\langle Y = y, Y = y' \rangle$ satisfying CC1--CC4 such that $(X = x) \subset (Y = y)$ and $(X = x') \subset (Y = y')$.
			It follows that $\langle X = x, X = x' \rangle$ could not be maximal according to CE4'.
			This contradicts the premise that $\langle X = x, X = x' \rangle$ satisfies CE1'--CE4' {under} EX1--EX4.
		\end{description}
		\item[CE3] This is satisfied by CE1' and CE2' via EX3.
		Suppose $\langle X = x, X = x' \rangle$ is not minimal.
		This implies there is some $\langle Y = y, Y = y' \rangle$ satisfying CE2 such that $(Y = y) \subset (X = x)$ and  $(Y = y') \subset (X = x')$.
		It follows that $X = x$ could not be minimal according to CE1' via EX3 and/or that $X = x'$ could not be minimal according to CE2' via EX3.
		This contradicts the premise that $\langle X = x, X = x' \rangle$ satisfies CE1'--CE4' {under} EX1--EX4.
		\item[CE4]
		\begin{description}
			\item[(a)] This is satisfied directly by CE1' via EX4.
			\item[(b)] This is satisfied directly by CE2' via EX4. \qed
		\end{description}
	\end{description}
\end{proof}

\begin{theoremrecall}{\ref{thm:original}}
	The Miller definition is equivalent to Definition~\ref{def:modular} {under} original HP explanations if $(M, u) \models \neg \varphi \lor \neg \psi$ for each $u \in K$.
\end{theoremrecall}

\begin{proof}
	This follows directly from Lemma~\ref{lem:original} (original HP contrastive explanations). \qed
\end{proof}

\subsection*{Modified HP explanations}

\begin{lemma}\label{lem:modified:basic}
	CH1--CH3 is equivalent to CE1'--CE4' {under} EX1'--EX3' if $(M, u) \models \neg \varphi \lor \neg \psi$ for each $u \in K$.
\end{lemma}

\begin{proof}
	($\rightarrow$) If $\langle X = x, X = x' \rangle$ satisfies CH1--CH3 for $\langle \varphi, \psi \rangle$ in $M$ relative to $K$ then $\langle X = x, X = x' \rangle$ satisfies CE1'--CE4' {under} EX1'--EX3' in $M$ relative to $K$:
	\begin{description}
		\item[CE1'] This is satisfied if there is some conjunction $Y = y$ that satisfies EX1'--EX3' for $\varphi$ in $M$ relative to $K$ such that $(X = x) \subseteq (Y = y)$:
		\begin{description}
			\item[EX1'(a)] This is satisfied by CH1(a) via CC1 combined with the premise of this lemma.
			If CH1(a) is satisfied then there must be some conjunct $X_{i} = x_{i}$ that is part of an actual cause of $\varphi$ in $(M, u)$ for each $u \in K$ such that $(M, u) \models (X = x) \land \varphi \land \neg \psi$.
			Moreover, if $(M, u) \models (X = x) \land \varphi \land \neg \psi$ and $(M, u) \models \neg \varphi \lor \neg \psi$ it cannot be that $(M, u) \models (X = x) \land \varphi \land \psi$, so the previous sentences holds for each $u \in K$ such that $(M, u) \models (X = x) \land \varphi$.
			\item[EX1'(b)] This is satisfied by the first clause of CH1(b).
			Simply take $Y = y$ as $(X = x) \land (S = s)$ from CH1(b).
			\item[EX2'] As in the proof for CE1' via EX3 in Lemma~\ref{lem:original}.
			\item[EX3'] This is satisfied directly by CH3.
		\end{description}
		\item[CE2'] The proof is the same as for CE1' except that $W \leftarrow w$ is the intervention mentioned by CH1(a) via CC3 and by CH1(b).
		\item[CE3'] This is satisfied directly by CH1(c).
		\item[CE4'] As in the proof for Lemma~\ref{lem:original}.
	\end{description}

	($\leftarrow$) If $\langle X = x, X = x' \rangle$ satisfies CE1'--CE4' {under} EX1'--EX3' for $\langle \varphi, \psi \rangle$ in $M$ relative to $K$ then $\langle X = x, X = x' \rangle$ satisfies CH1--CH3 in $M$ relative to $K$:
	\begin{description}
		\item[CH1]
		\begin{description}
			\item[(a)] This is satisfied if there is some pair of conjunctions $\langle Y = y, Y = y' \rangle$ that satisfies CC1--CC5 for $\langle \varphi, \psi \rangle$ in $(M, u)$ for each $u \in K$ such that $(M, u) \models (X = x) \land \varphi \land \neg \psi$ where $(X = x) \cap (Y = y) \ne \emptyset$ and $(X = x') \cap (Y = y') \ne \emptyset$:
			\begin{description}
				\item[CC1] This is satisfied directly by CE1' via EX1'(a).
				\item[CC2] This is satisfied by CE1' via EX1'(a) and the premise of this lemma.
				If $(M, u) \models (X = x) \land \varphi$ and $(M, u) \models \neg \varphi \lor \neg \psi$ then it must be that $(M, u) \models \neg \psi$ since $((X = x) \land \varphi) \land (\neg \varphi \lor \neg \psi) \vdash \neg \psi$.
				\item[CC3] This is satisfied directly by CE2' via EX1'(a).
				\item[CC4] This is satisfied directly by CE3'.
				\item[CC5] This is irrelevant since we are only interested in a single variable.
			\end{description}
			\item[(b)] This is satisfied by CE1' and CE2' via EX1'(b).
			For the first clause, simply take $(X = x) \land (S = s)$ from CH1(b) as an explanation containing the partial explanation referenced by CE1' via EX1(b).
			For the second clause, simply take $(X = x') \land (S = s')$ from CH1(b) as an explanation containing the partial explanation referenced by CE2' via EX1(b).
			\item[(d)] As in the proof for CE2 via CC5 in Lemma~\ref{lem:original}.
		\end{description}
		\item[CH2] As in the proof for CE3 in Lemma~\ref{lem:original}.
		\item[CH3] This is satisfied by CE1' via EX3' combined with the premise of this lemma.
		If $(M, u) \models (X = x) \land \varphi$ and $(M, u) \models \neg \varphi \lor \neg \psi$ then it must be that $(M, u) \models (X = x) \land \varphi \land \neg \psi$ since $((X = x) \land \varphi) \land (\neg \varphi \lor \neg \psi) \vdash (X = x) \land \varphi \land \neg \psi$. \qed
	\end{description}
\end{proof}

\begin{lemma}\label{lem:modified:non-trivial}
	CH1--CH4 is equivalent to CE1'--CE4' {under} EX1'--EX4' if $(M, u) \models \neg \varphi \lor \neg \psi$ for each $u \in K$.
\end{lemma}

\begin{proof}
	($\rightarrow$) If $\langle X = x, X = x' \rangle$ satisfies CH1--CH4 for $\langle \varphi, \psi \rangle$ in $M$ relative to $K$ then $\langle X = x, X = x' \rangle$ satisfies CE1'--CE4' {under} EX1'--EX4' in $M$ relative to $K$:
	\begin{description}
		\item[CE1'] This is satisfied if there is some conjunction $Y = y$ that satisfies EX1'--EX4' for $\varphi$ in $M$ relative to $K$ such that $(X = x) \subseteq (Y = y)$:
		\begin{description}
			\item[EX1'--EX3'] As in the proof for Lemma~\ref{lem:modified:basic}.
			\item[EX4'] This is satisfied directly by CE1' via EX4.
		\end{description}
		\item[CE2'--CE4'] As in the proof for Lemma~\ref{lem:modified:basic}.
	\end{description}

	($\leftarrow$) If $\langle X = x, X = x' \rangle$ satisfies CE1'--CE4' {under} EX1'--EX4' for $\langle \varphi, \psi \rangle$ in $M$ relative to $K$ then $\langle X = x, X = x' \rangle$ satisfies CH1--CH4 in $M$ relative to $K$:
	\begin{description}
		\item[CH1--CH3] As in the proof for Lemma~\ref{lem:modified:basic}.
		\item[CH4]
		\begin{description}
			\item[(a)] This is satisfied directly by CE1' via EX4'.
			\item[(b)] This is satisfied directly by CE2' via EX4'. \qed
		\end{description}
	\end{description}
\end{proof}

\begin{theoremrecall}{\ref{thm:modified}}
	Definition~\ref{def:modified} is equivalent to Definition~\ref{def:modular} {under} modified HP explanations if $(M, u) \models \neg \varphi \lor \neg \psi$ for each $u \in K$.
\end{theoremrecall}

\begin{proof}
	This follows directly from Lemma~\ref{lem:modified:basic} (modified HP contrastive explanations) and Lemma~\ref{lem:modified:non-trivial} (non-trivial modified HP contrastive explanations). \qed
\end{proof}

\subsection*{Borner explanations}

\begin{lemma}\label{lem:borner:potential}
	CB1--CB4 is equivalent to CE1'--CE4' {under} E1--E4 if $(M, u) \models \neg \varphi \lor \neg \psi$ for each $u \in K$.
\end{lemma}

\begin{proof}
	($\rightarrow$) If $\langle X = x, X = x' \rangle$ satisfies CB1--CB4 for $\langle \varphi, \psi \rangle$ in $M$ relative to $K$ then $\langle X = x, X = x' \rangle$ satisfies CE1'--CE4' {under} E1--E4 in $M$ relative to $K$:
	\begin{description}
		\item[CE1'] This is satisfied if there is some conjunction $Y = y$ that satisfies E1--E4 for $\varphi$ in $M$ relative to $K$ such that $(X = x) \subseteq (Y = y)$:
		\begin{description}
			\item[E1--E2]
			\begin{description}
				\item[(a)] This is satisfied by CB1--CB2(a) via CC1.
				If CB1--CB2(a) is satisfied then $(X = x) \land (S = s)$ must be a partial sufficient cause of $\varphi$ in $(M, u)$ for each $u \in K$ such that $(M, u) \models (X = x)$.
				Since a partial cause is defined as a subset of a cause, there must be some $Y = y$ that is a sufficient cause of $\varphi$ in $(M, u)$ for each $u \in K$ such that $(M, u) \models (X = x)$ where $(X = x) \land (S = s) \subseteq (Y = y)$.
				\item[(b)] This is satisfied directly by CB1--CB2(b).
			\end{description}
			\item[E3] This is satisfied directly by CB3.
			\item[E4] This is satisfied directly by CB4(a). 
		\end{description}
		\item[CE2'] The proof is the same as for CE1' except that $W \leftarrow w$ is the intervention mentioned by CB1--CB2(a) via CC3, by CB1--CB2(c), and by CB4(b).
		\item[CE3'] This is satisfied directly by CB1--CB2(a) via CC4.
		\item[CE4'] As in the proof for Lemma~\ref{lem:original}.
	\end{description}

	($\leftarrow$) If $\langle X = x, X = x' \rangle$ satisfies CE1'--CE4' {under} E1--E4 for $\langle \varphi, \psi \rangle$ in $M$ relative to $K$ then $\langle X = x, X = x' \rangle$ satisfies CB1--CB4 in $M$ relative to $K$:
	\begin{description}
		\item[CB1--CB2]
		\begin{description}
			\item[(a)] This is satisfied if there is some pair of possibly empty conjunctions $\langle S = s, S = s' \rangle$ where $\langle X = x \land S = s, X = x' \land S = s' \rangle$ satisfies CC1--CC5 for $\langle \varphi, \psi \rangle$ in $(M, u)$ for each $u \in K$ such that $(M, u) \models (X = x)$, except that references to actual causes in CC1--CC5 are replaced by sufficient causes:
			\begin{description}
				\item[CC1] This is satisfied directly by CE1' via E1--E2(a).
				\item[CC2] This is satisfied by CE1' via E1--E2(a) and SC1 combined with the premise of this lemma.
				SC1 requires $(M, u) \models (X = x) \land \varphi$ and the premise requires $(M, u) \models \neg \varphi \lor \neg \psi$ so it must be that $(M, u) \models \neg \psi$ since $((X = x) \land \varphi) \land (\neg \varphi \lor \neg \psi) \vdash \neg \psi$.
				\item[CC3] This is satisfied directly by CE2' via E1--E2(a).
				\item[CC4] This is satisfied directly by CE3'.
				\item[CC5] As in the proof for CE2 via CC5 in Lemma~\ref{lem:original}.
			\end{description}
			\item[(b)] This is satisfied directly by CE1' via E1--E2(b).
			\item[(c)] This is satisfied directly by CE2' via E1--E2(b).
		\end{description}
		\item[CB3] This is satisfied by CE1' via E3 combined with the premise of this lemma.
		E3 requires $(M, u) \models (X = x) \land \varphi$ for some $u \in K$ and the premise requires $(M, u) \models \neg \varphi \lor \neg \psi$ so it must be that $(M, u) \models (X = x) \land \varphi \land \neg \psi$ since $((X = x) \land \varphi) \land (\neg \varphi \lor \neg \psi) \vdash (X = x) \land \varphi \land \neg \psi$.
		\item[CB4]
		\begin{description}
			\item[(a)] This is satisfied directly by CE1' via E4.
			\item[(b)] This is satisfied directly by CE2' via E4. \qed
		\end{description}
	\end{description}
\end{proof}

\begin{lemma}\label{lem:borner:actual}
	CB1--CB3, CB5 is equivalent to CE1'--CE4' {under} E1--E3, E5 if $(M, u) \models \neg \varphi \lor \neg \psi$ for each $u \in K$.
\end{lemma}

\begin{proof}
	($\rightarrow$) If $\langle X = x, X = x' \rangle$ satisfies CB1--CB3, CB5 for $\langle \varphi, \psi \rangle$ in $M$ relative to $K$ then $\langle X = x, X = x' \rangle$ satisfies CE1'--CE4' {under} E1--E3, E5 in $M$ relative to $K$:
	\begin{description}
		\item[CE1'] This is satisfied if there is some conjunction $Y = y$ that satisfies E1--E3, E5 for $\varphi$ in $M$ relative to $K$ such that $(X = x) \subseteq (Y = y)$:
		\begin{description}
			\item[E1--E3] As in the proof for Lemma~\ref{lem:borner:potential}.
			\item[E5] This is satisfied directly by CB5.
		\end{description}
		\item[CE2'--CE4'] As in the proof for Lemma~\ref{lem:borner:potential}.
	\end{description}

	($\leftarrow$) If $\langle X = x, X = x' \rangle$ satisfies CE1'--CE4' {under} E1--E3, E5 for $\langle \varphi, \psi \rangle$ in $M$ relative to $K$ then $\langle X = x, X = x' \rangle$ satisfies CB1--CB3, CB5 in $M$ relative to $K$:
	\begin{description}
		\item[CB1--CB3] As in the proof for Lemma~\ref{lem:borner:potential}.
		\item[CB5] This is satisfied by CE1' via E5 and the premise of this lemma.
		E5 requires $(M, u) \models (X = x) \land \varphi$ for each $u \in K$ and the premise requires $(M, u) \models \neg \varphi \lor \neg \psi$ so it must be that $(M, u) \models (X = x) \land \varphi \land \neg \psi$ since $((X = x) \land \varphi) \land (\neg \varphi \lor \neg \psi) \vdash (X = x) \land \varphi \land \neg \psi$. \qed
	\end{description}
\end{proof}

\begin{lemma}\label{lem:borner:parsimonious}
	CB1--CB4, CB6 is equivalent to CE1'--CE4' {under} E1--E4, E6 if $(M, u) \models \neg \varphi \lor \neg \psi$ for each $u \in K$.
\end{lemma}

\begin{proof}
	($\rightarrow$) If $\langle X = x, X = x' \rangle$ satisfies CB1--CB4, CB6 for $\langle \varphi, \psi \rangle$ in $M$ relative to $K$ then $\langle X = x, X = x' \rangle$ satisfies CE1'--CE4' {under} E1--E4, E6 in $M$ relative to $K$:
	\begin{description}
		\item[CE1'] This is satisfied if there is some conjunction $Y = y$ that satisfies E1--E4, E6 for $\varphi$ in $M$ relative to $K$ such that $(X = x) \subseteq (Y = y)$:
		\begin{description}
			\item[E1--E4] As in the proof for Lemma~\ref{lem:borner:potential}.
			\item[E6] As in the proof for CE1' via EX3 in Lemma~\ref{lem:original}.
		\end{description}
		\item[CE2'--CE4'] As in the proof for Lemma~\ref{lem:borner:potential}.
	\end{description}

	($\leftarrow$) If $\langle X = x, X = x' \rangle$ satisfies CE1'--CE4' {under} E1--E4, E6 for $\langle \varphi, \psi \rangle$ in $M$ relative to $K$ then $\langle X = x, X = x' \rangle$ satisfies CB1--CB4, CB6 in $M$ relative to $K$:
	\begin{description}
		\item[CB1--CB4] As in the proof for Lemma~\ref{lem:borner:potential}.
		\item[CB6] As in the proof for CE3 in Lemma~\ref{lem:original}. \qed
	\end{description}
\end{proof}

\begin{theoremrecall}{\ref{thm:borner}}
	Definition~\ref{def:borner} is equivalent to Definition~\ref{def:modular} {under} Borner explanations if $(M, u) \models \neg \varphi \lor \neg \psi$ for each $u \in K$.
\end{theoremrecall}

\begin{proof}
	This follows directly from Lemma~\ref{lem:borner:potential} (potential Borner contrastive explanations), Lemma~\ref{lem:borner:actual} (actual Borner contrastive explanations), and Lemma~\ref{lem:borner:parsimonious} (parsimonious potential Borner contrastive explanations). \qed
\end{proof}

\end{document}